%% file: main.tex
\documentclass[9pt]{article}
\usepackage[T1]{fontenc}
\usepackage{lmodern}
\usepackage[numbers, square]{natbib}

\usepackage[verbose=true,letterpaper]{geometry}
\AtBeginDocument{
  \newgeometry{
    textheight=9in,
    textwidth=5.5in,
    top=1in,
    headheight=12pt,
    headsep=25pt,
    footskip=30pt
  }
}

\usepackage[utf8]{inputenc} 
\usepackage[T1]{fontenc}    
\usepackage{hyperref}       
\usepackage{url}            
\usepackage{dsfont}
\usepackage{booktabs}       
\usepackage{amsfonts,amssymb}       
\usepackage{nicefrac}       
\usepackage{microtype}      
\usepackage{amsthm,amsfonts,amsmath}
\usepackage{graphicx}
\usepackage{enumitem}
\usepackage{subfig}
\usepackage[algoruled,boxed,lined]{algorithm2e}
\usepackage{color}

\hypersetup{
    colorlinks=true,
    linkcolor=blue,
    filecolor=blue,      
    urlcolor=blue,
    citecolor=black,
}

\newcommand{\centre}{centre}

\renewcommand{\cite}{\citet}

\newcommand{\eps}{\varepsilon}
\newcommand{\calA}{\mathcal{A}}
\newcommand{\calM}{\mathcal{M}}

\newcommand{\calR}{\mathcal{R}}
\newcommand{\calU}{\mathcal{U}}
\newcommand{\opt}{\mathcal{C}}
\newcommand{\loss}{\ell}
\newcommand{\tS}{\tilde{S}}

\newcommand{\calH}{\mathcal{H}}

\newcommand{\twonorm}[1]{\Vert #1 \Vert_2}

\theoremstyle{plain}
\newtheorem{theorem}{Theorem}[section]
\newtheorem{lemma}[theorem]{Lemma}
\newtheorem{corollary}[theorem]{Corollary}

\newcounter{sideremark}

\usepackage{setspace}

\newcommand{\reals}{\mathbb{R}}
\newcommand{\NP}{\mathsf{NP}}
\newcommand{\APX}{\mathsf{APX}}
\newcommand{\poly}{\mathrm{poly}}

\newcommand{\apxsol}{\mathcal{S}^{(t)}}
\newcommand{\setS}[1]{S^{(#1)}}
\newcommand{\charALG}{$\calA_{cop}$}
\newcommand{\sample}{A}
\newcommand{\subcoreset}{Q}
\newcommand{\chenCoreset}{\subcoreset}

\newcommand{\reffig}{\ref}
\newcommand{\algref}{\ref}
\newcommand{\refsec}{\ref}
\newcommand{\bbE}{\mathbb{E}}

\newcommand{\kmeans}{$k$-means}
\newcommand{\kmeanspp}{$k$-means\texttt{++}}

\newcommand{\MNIST}{MNIST}
\newcommand{\MWUA}{\textsf{MWUA}}
\newcommand{\NPHARD}{\textsf{NP}-hard}
\newcommand{\MTMW}{\textsf{MTMW}} 
\newcommand{\HRD}{\textsf{HRD}}
\newcommand{\FTL}{\textsf{FTL}}
\newcommand{\MWUAFTL}{\textsf{MWUA-FTL}}
\renewcommand{\c}[1]{C_{#1}}
\newcommand{\cstar}[1]{\c{#1}^{*}}

\title{Online $k$-means Clustering}
\author{%
  Vincent Cohen-Addad \\
  CNRS\\
  \texttt{vcohenad@gmail.com} \\
  \and
  Benjamin Guedj \\
  Inria and University College London \\
  \texttt{benjamin.guedj@inria.fr} \\
  \and
  Varun Kanade \\
  University of Oxford \\
  \texttt{varunk@cs.ox.ac.uk} \\
  \and
  Guy Rom \\
  University of Oxford \\
  \texttt{guy.rom@cs.ox.ac.uk} \\
}

\begin{document}

\maketitle

\begin{abstract}
	\input{files/abstract}
\end{abstract}

\section{Introduction}
\label{sec:intro}
\input{files/intro}

\section{Basic Results}
\label{sec:basic-results}
\vspace{-2mm}
\subsection{Preliminaries and Notation}
\label{sec:prelims}
\input{files/prelims.tex}

\vspace{-2mm}
\subsection{\MWUA{} with Grid Discretization}
\label{sec:mwua}
\input{files/mwua}
\vspace{-5mm}
\subsection{Lower Bound}
\label{sec:off2on}
\input{files/off2on.tex}

\subsection{Follow The Leader-- Linear Regret Worst Case}
\label{sec:FTL_counter}
\input{files/ftl_counter.tex}

\subsection{Follow The Leader-- Sublinear Regret on Natural Datasets}
\label{sec:FTL_sub}
\input{files/ftl_sub.tex}

\section{Approximate Regret Minimization}
\label{sec:approx-regret}
\input{files/approx_regret}

\section{Discussion}
\input{files/discussion}

\bibliographystyle{unsrtnat}
\bibliography{main}

\appendix

\section{Proofs of Results from Section~\ref{sec:basic-results}}
\label{sec:appendix-basic}
\input{files/basic_proofs}

\section{Proofs of Results from Section~\ref{sec:approx-regret}}
\label{sec:appendix-approx}
\input{files/apx_regret_proofs.tex}

\end{document}

%% file: files/abstract.tex
We study the problem of online clustering where a clustering algorithm has to assign a new point that arrives to one of $k$ clusters. The specific formulation we use is the \kmeans{} objective: At each time step the algorithm has to maintain a set of $k$ candidate centers and the loss incurred is the squared distance between the new point and the closest center. The goal is to minimize regret with respect to the best solution to the \kmeans{} objective ($\opt$) in hindsight. We show that provided the data lies in a bounded region, an implementation of the Multiplicative Weights Update Algorithm (\MWUA{}) using a discretized grid achieves a regret bound of $\tilde{O}(\sqrt{T})$ in expectation. We also present an online-to-offline reduction that shows that an efficient no-regret online algorithm (despite being allowed to choose a different set of candidate \centre{}s at each round) implies an offline efficient algorithm for the \kmeans{} problem. In light of this hardness, we consider the slightly weaker requirement of comparing regret with respect to $(1 + \epsilon) \opt$ and present a no-regret algorithm with runtime $O\left(T(\mathrm{poly}(\log(T),k,d,1/\epsilon)^{k(d+O(1))}\right)$. Our algorithm is based on maintaining an incremental coreset and an adaptive variant of the \MWUA{}.  We show that na\"{i}ve online algorithms, such as \emph{Follow The Leader}, fail to produce sublinear regret in the worst case. We also report preliminary experiments with synthetic and real-world data.

%% file: files/intro.tex
Clustering algorithms are one of the main tools of unsupervised learning
and often form a key part of a data analysis pipeline. Unlabeled data is
ubiquitous in the real world and discovering structure in such data is
essential in many online applications. The focus of this work is on the
\emph{online} setting where data elements arrive one at a time and need to be
assigned to a cluster (either new or existing) without the benefit of having
observed the entire sequence. While several objective functions for clustering
exist, in our work we will focus on the \kmeans{} objective. Most of our
results can be easily generalized to most \centre{}-based objectives.

The analysis of online algorithms comes in two flavours involving bounding
either the \emph{competitive ratio}, or the \emph{regret}. 
The online algorithm makes irrevocable decisions and its
performance is measured by the value the objective function. The competitive
ratio is the ratio between the value achieved by the online algorithm and the
best offline solution (for minimization problems). In the case of clustering,
without strong assumptions on the aspect-ratio of the instance no algorithms
with non-trivial bounds on the competitive ratio can be designed.
In \emph{regret analysis}, the difference between the
value of the objective function of the online algorithm and the best offline
solution (in hindsight) is sought to be bounded by a function that grows
sublinearly with the number of data elements. We consider \emph{regret
analysis} in this paper.

More precisely, in the case of online \kmeans{} clustering, the online
algorithm at time $t$ maintains a set of $k$ candidate cluster \centre{}s, $C_t =
\{ c_{t,1}, \ldots, c_{t, k} \}$ before observing the datum $x_t$ that arrives
at time $t$. The loss incurred by the algorithm at time $t$ is $\ell_t(C_t,
x_t) = \min_{c \in C_t} \twonorm{x_t - c}^2$. The \emph{regret} is the
difference between the cumulative loss of the algorithm over $T$ time steps and
the optimal fixed solution in hindsight, i.e. 
\[
	\sum_{t=1}^T \ell(C_t, x_t) - \min_{C: |C| = k} \sum_{t=1}^T \min_{c \in C} \twonorm{x_t - c}^2.
\]

\subsection{Our Contributions}

We consider the setting where the data $x_t$ all lie in the unit box in $[0,
1]^d \subset \reals^d$. We summarise our contributions below.
\begin{itemize}[itemsep=0.2em, leftmargin=1em, label=-]
	\item A multiplicative weight-update algorithm (\MWUA{}) over sets of size $k$ of candidate \centre{}s drawn from a uniform grid over $[0, 1]^d$ achieves expected regret $\tilde{O}(\sqrt{T})$; here the $\tilde{O}$ notation hides factors that are poly-logarithmic in $T$ and polynomial in $k$ and $d$. The algorithm and its analysis is along standard lines and the algorithm is computationally inefficient. Nevertheless, this algorithm establishes that information-theoretically achieving $\tilde{O}(\sqrt{T})$ regret is possible. 
	\item We provide an online-to-offline reduction that shows that any online
		algorithm that runs in time $f(t, k, d)$ at time $t$, yields an offline
		algorithm that solves the \kmeans{} problem to additive accuracy $\epsilon$ in
		time polynomial in $n$, $k$, $d$, $1/\epsilon$, $f(t, k, d)$. In
		particular, for an offline instance of \kmeans{} with $n$ point in a
		bounded region with $\opt \geq 1/\mathrm{poly}(n)$, an online algorithm
		with polynomial run time would yield a fully poly-time approximation
		scheme. We note that there exist hard instances for \kmeans{} for which
		$\opt \geq 1/\poly(n)$ and that it is known that \kmeans{} is
		$\APX$-hard \citep{awasthi2015hardness}. Furthermore, all known (approximation) FPT algorithms for
		\kmeans{} are exponential in at least one of the two parameters $k$ and
		$d$. This suggests that we need to relax performance requirements for
		efficient algorithms.
	\item We consider a weaker notion of regret called $(1 + \epsilon)$-regret. Let $\opt$ denote the loss of the best solution in hindsight and $L_T$ the cumulative loss of the algorithm; in the definition of regret, instead of $L_T - \opt$, we consider $L_T - (1 + \epsilon)\opt$. With this notion of regret, we provide an algorithm that achieves $(1+\epsilon)$-regret $\tilde{O}(\sqrt{T})$ and runs in time $O\left(T(d^{1/2} k^2 \epsilon^{-2}\log(T))^{k(d+O(1))}\right)$.
	\item Finally, we consider online algorithms which have oracle access to \kmeans{} solver. For instance, this allows the us to implement \emph{follow the leader}. We show that there exists a sequence of examples for which follow the leader has \emph{linear} regret. We show that this construction indeed results in linear regret in simulations. We observe that \FTL{}~(using \kmeans{}++ as a proxy for oracle) works rather well on real-world data. 
\end{itemize}

\subsection{Related Work}

Clustering has been studied from various perspectives, e.g. combinatorial optimization, probabilistic modelling, and there are several widely used algorithms such as Lloyd's local search algorithm with \kmeans{}++ seeding, spectral methods, the EM algorithm, etc. 
We will restrict discussion mainly to the clustering as combinatorial optimization viewpoint. The \kmeans{} objective is one of the family of \centre{}-based objectives which uses the squared distance to the \centre{} as a measure of variance. Framed as an optimization problem, the problem is $\NP$-complete. As a result, theoretical work has focused on approximation and FPT algorithms. 

A related model to the online framework is the streaming model. As in the online model, the data is received one at a time. The focus in the streaming framework is to have extremely low memory footprint and the algorithm is only required to propose a solution once the stream has been exhausted. In contrast, in the online setting the learning algorithm has to make a decision at each time step and incur a corresponding loss. Coresets are widely used in computational geometry to obtain approximation algorithms. A coreset for \kmeans{} is a mapping of the original data to a subset of the data, along with a weight function, such that the \kmeans{} cost of partitions of the data is preserved up to some small error using the given mapping and weights.

Online learning with experts and related problems have been widely studied, see
e.g.~\cite{cesa2006prediction} and references therein. The
\emph{Multiplicative Weight Update Algorithm (\MWUA{})} is a widely studied
algorithm that may be used for regret minimization in the prediction with
expert advice setting. It maintains a distribution over experts that changes as
new data points arrive, which is used to sample an advice of some expert to be
used as the next prediction, resulting in low regret. For a thorough survey see
\cite{arora2012multiplicative}. \FTL{} is a simple online algorithm that always predicts
the best solution for the data witnessed so far. It is known to admit low regret for some problems,
namely, strongly convex objectives \citep{shalev2012online}. Variants of \FTL{}
that optimize a regularized objective have been successfully utilized for a wider range of settings. 

The closest related work to ours is the work of \cite{das2008onlclas}. He defines the evaluation framework we use here, and presents a na\"ive greedy algorithm to adress it, with no analysis. In addition, he combines algorithms by \cite{charikar2004incremental} and \cite{beygelzimer2006cover} that together maintain a set of constant approximations of the $k$-\centre{} objective at any time, for a range of values for $k$.

\cite{choromanska2012online} study the online clustering problem in the presence of experts. The experts are batch clustering algorithms that output the \centre{} closest to the next point at each step (hence provide implicit information on the next point in the stream). Using experts that have approximation guarantees for the batches they obtain an approximate regret guarantee of $\log(T)$ for the stream. Our setting differs in that we must commit to the next cluster \centre{}s strictly before the next point in the stream is observed, or any implicit information about it.

\cite{li2018quasi} provide a generalized Bayesian adaptive online clustering algorithm (built upon the PAC-Bayes theory). They describe a Gibbs Sampling procedure of $O(k)$ \centre{}s and prove it has a minimax sublinear regret. They present a reversible jump MCMC process to sample these \centre{}s with no theoretical mixing time analysis.

\cite{moshkovitz2019unexpected} studies a similar problem that considers only data points as candidate cluster \centre{}s, and the offline solution is defined similarly (an $\ell_2$ analog of the $k$-medoids problem). Furthermore, the algorithm starts with an empty set of cluster \centre{}s and must output an incremental solution-- cluster centers are only added to the set, and this is done in a streaming fashion. The loss is measured in hindsight, using the final set of cluster \centre{}s. They provide tight bounds for both adversarial and randomly ordered streams, and show that knowing the length of the stream reduces the amount of \centre{}s required to obtain a constant approximation by a logarithmic factor.

\cite{liberty2016algorithm} handle a different definition of \emph{online}-- the data points are labeled in an online fashion but the loss is calculated according to the centroids of the final clusters. They allow $O(k\cdot f(n,\gamma))$ clusters where $\gamma$ is the aspect ratio of the data (the ratio between the diameter of the set and the closest distance between two points), and guarantee a constant competitive ratio when compared to the best $k$ clusters in hindsight.

\cite{meyerson2001online} studies \emph{online facility location}, where one maintains a set of facilities at each time step, and suffers a loss which is the distance of the new point to the closest facility. Once a facility is located, it cannot be moved, and placing a new one incurs a loss. Meyerson presents an algorithm that has a constant competitive ratio on randomly ordered streams. For adversarial order he presents an algorithm with $\log(T)$ competitive ratio and provides a lower bound showing no algorithm can have a constant competitive ratio.

%% file: files/prelims.tex
To precisely define the online clustering problem we will consider points in the unit box $[0, 1]^d \subseteq \reals^d$. The point
arriving at time $t$ will be denoted by $x_t$ and we use $X_{1:t-1}$ to denote
the data received before time $t$. The learning algorithm must output a set
$\c{t}$ of $k$ candidate \centre{}s using only $X_{1:t-1}$. We refer to the set of
all the candidate \centre{}s an algorithm is considering as \emph{sites}. The loss incurred by the
algorithm at time $t$ is $\loss(\c{t}, x_t) = \min_{c \in \c{t}} \twonorm{x_t -
c}^2$. The total loss of an algorithm up to time step $t$ is denoted by $L_{t}
= \sum_{\tau=1}^t \loss(\c{\tau}, x_\tau)$. The loss of the best \kmeans{} solution in
\emph{hindsight} after $T$ steps is denoted by $\opt$. The regret is
defined as $L_T - \opt$. Several of the algorithms we consider will pick cluster \centre{}s from a
constrained set; we sometimes refer to any set of $k$ sites from such a
constrained set as an \emph{expert}. We define the $(1+\epsilon)$-approximate
regret as $L_T - (1 + \epsilon) \opt$.

The loss of the weighted \kmeans{} problem is defined similarly, 
given a weight function $\omega:X\to\reals_+$, as 
$\loss(\boldsymbol{\mu}, X) = \sum_{x\in X} \omega(x)\min_{\mu \in \boldsymbol{\mu}} \twonorm{x - \mu}^2$. 

We denote the best \kmeans{} solution, i.e. best $k$ cluster \centre{}s, for $X_{1:t-1}$ by $\cstar{t}$, hence
$\cstar{T+1}$ is the best \kmeans{} solution in hindsight. In our setting, the
\emph{Follow-The-Leader} (\FTL{}) algorithm simply picks $\cstar{t}$ at time $t$. 
We use
$\tilde{O}$ to suppress factors that are poly-logarithmic in $T$ and polynomial
in $k$ and $d$.

%% file: files/mwua.tex
While the multiplicative weight update algorithm (\MWUA{}) is very widely
applicable, there are a couple of difficulties when it comes to applying it to
our problem. In order to obtain a finite set of experts, we consider
\emph{sites} obtained by a $\delta$-grid of $[0, 1]^d$. In order to obtain
regret bounds that are $\tilde{O}(T^{1-\alpha})$ for $0 < \alpha \leq 1/2$ (typically
$\alpha = 1/2$), we need to choose $\delta = T^{-\alpha/2}$; this means that the number of experts is exponentially large in $k$ and $d$ and polynomially large in $T$. However, since the regret of \MWUA{} only has logarithmic dependence on the number of experts, this mainly incurs a computational, rather than statistical cost.%
\footnote{There appears to be some statistical cost in that we are only able to prove bounds on expected regret.}
In Section~\ref{sec:approx-regret}, we develop a more \emph{data-adaptive}
version of the \MWUA{}. This allows us to significantly reduce the number of
\emph{sites} required---we don't put sites in location where there is no
data---but also requires a much more intricate analysis. The price we pay for
adaptivity and computational efficiency is that we are only able to get regret
bounds for $(1+\epsilon)$-regret. However, the results in
Section~\ref{sec:off2on} show that this is (under computational conjectures)
unavoidable.

\begin{theorem}
\label{thrm:eps-grid-MWUA}
Let $S = \{ i \delta ~|~ 0 \leq i \leq \delta^{-1} \}^d \subset \reals^d$ be
the set of \emph{sites} and let $\mathcal{E} = \{ C \subset S ~|~ |C| = k \}$
be the set of experts ($k$-\centre{}s chosen out of the sites). Then, for any $0 <
\alpha \leq 1/2$, with $\delta = T^{-\alpha/2}$, the \MWUA{}  with the expert set
$\mathcal{E}$ achieves regret $\tilde{O}(T^{1 - \alpha})$; the per round
running time is $O(T^{\alpha k d/2})$. 
\end{theorem}

%% file: files/off2on.tex
Given the disappointing runtime of the grid-\MWUA{} algorithm, one may wonder whether there is a way to avoid explicitly storing a weight for each of the exponentially many experts and speed-up the \MWUA{} algorithm. The following result gives evidence that it is unlikely that a significant speed-up is possible under complexity-theoretic assumptions. In particular, a consequence of Theorem~\ref{thrm:offline_reduction} is that for instances of \kmeans{} with data lying in a bounded region and $\opt \geq 1/\mathrm{poly}(n)$, a per-round polynomial time online algorithm would imply a fully polynomial-time approximation scheme. Recall that \kmeans{} is $\APX$-hard and current best known algorithms are exponential in at least one of the two parameters $k$ and $d$. 

\begin{theorem}
\label{thrm:offline_reduction}
Suppose there is an online \kmeans{} clustering algorithm $\mathcal{A}$ that achieves regret $\tilde{O}(T^{1 - \alpha})$ and runs at time $t$ in time $f(t, k, d)$. Then, for any $\epsilon > 0$, there is a randomized offline algorithm that given an instance of \kmeans{} outputs a solution with cost at most $\opt + \epsilon$ with constant probability and runs in time polynomial in $n, k, d, \frac{1}{\epsilon}, f(n, k, d)$.
\end{theorem}

%% file: files/ftl_counter.tex
The lower bound presented here does not imply anything on the regret guarantees of \emph{Follow-The-Leader (\FTL{})} that has orcale access to the best offline solution $\cstar{t}$ at each step. This section will show that \FTL{} incurs linear regret in the worst case, 
\begin{theorem}
	\label{thrm:ftl}
\FTL{} obtains $\Omega(T)$ regret in the worst case, for any fixed $k\ge 2$ and any dimension.
\end{theorem}
\begin{figure}
\centering
\subfloat[\label{fig:ftl_counter}\FTL{} Linear Regret, \MWUAFTL{} Sublinear Regret]{
\includegraphics[width=0.48\textwidth]{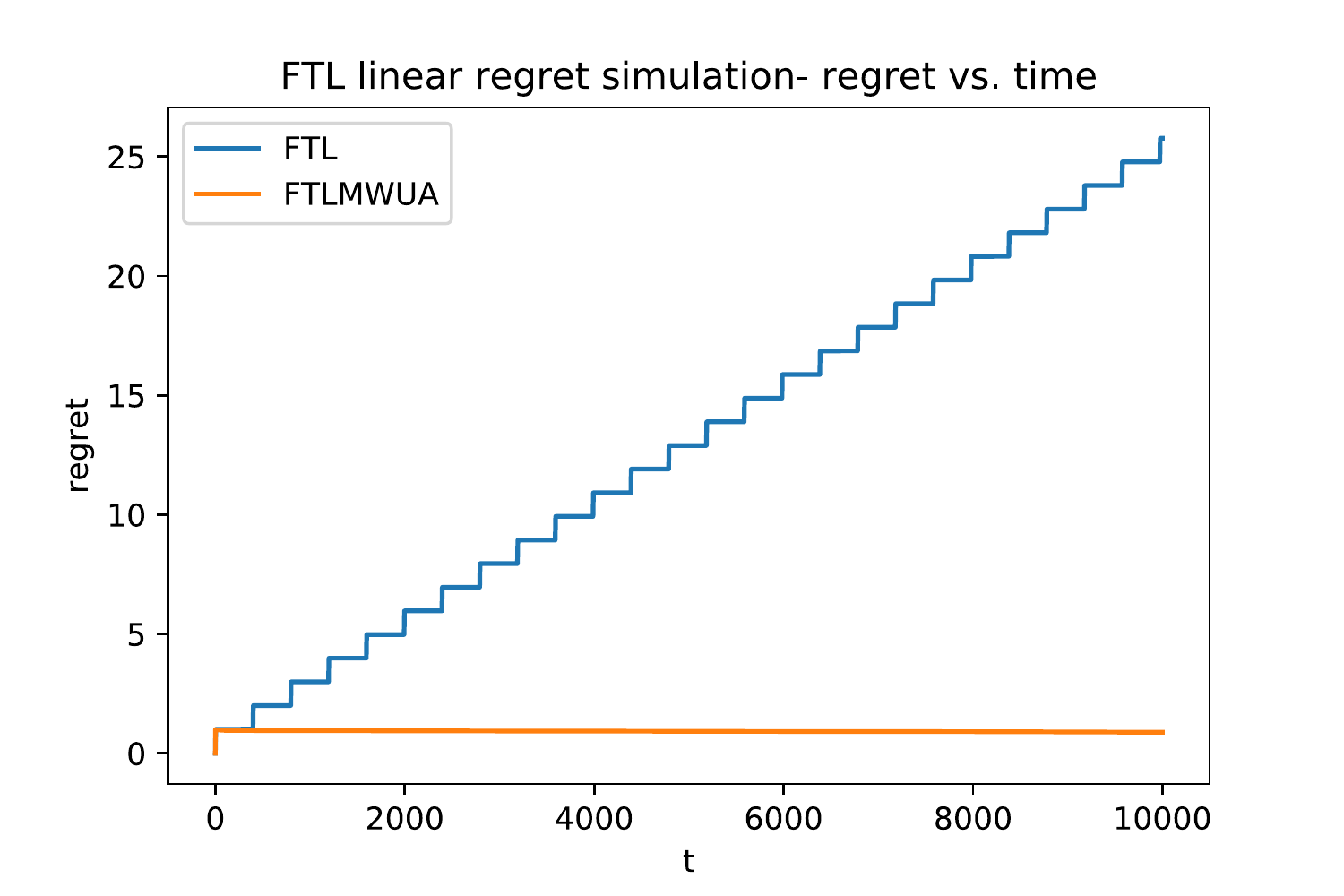}
}
\hfill
\subfloat[\label{fig:ftl_vs_sim_mnist} \FTL{} Regret--\MNIST{} and Gaussian Mixure Models]{
\includegraphics[width=0.48\textwidth]{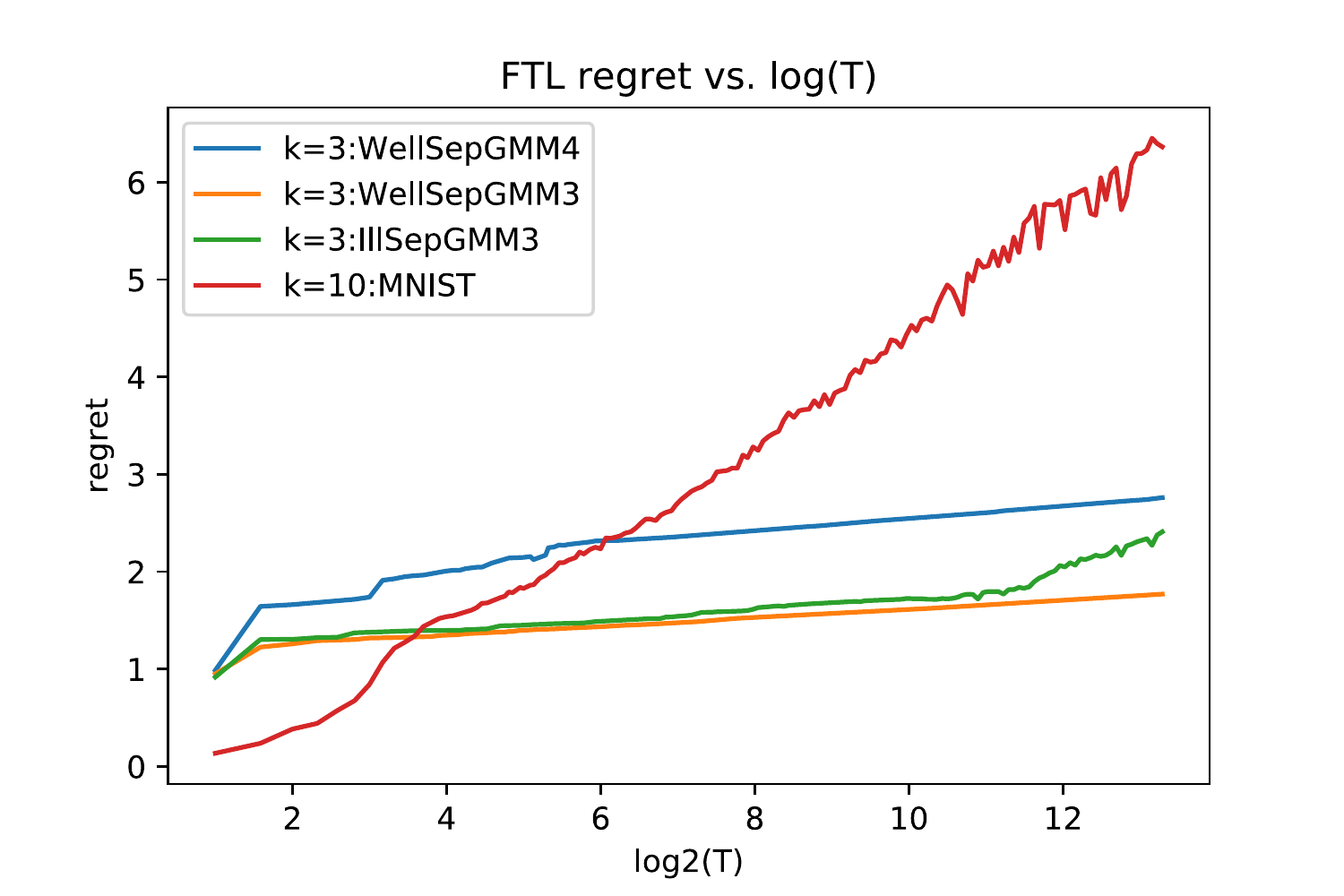}
}
\end{figure}

Figure \reffig{fig:ftl_counter} shows the regret of \FTL{} at any point in time if the stream halted at that point, which we refer to as the \emph{regret halted at $t$}, for a stream that was generated according to the scheme describe above with $\delta=0.1$, along with a \MWUA{} over the set of leaders $\{\cstar{t}\}_{t=1}^T$, where the \MWUA{} weight for any leader is calculated according to their historical loss, regardless of the time t when it was introduced to the expert set, i.e. $\forall t'>t:w^{(t')}_{\cstar{t}}=\prod_{\tau=1}^{t'-1}(1-\eta \loss(\cstar{t},x_{\tau}))$. The staircase-like line for \FTL{} is caused by the fact that the specific data used makes \FTL{} suffer a constant excess loss (w.r.t. the optimal solution) every several iterations, and negligent excess loss in the rest of the iterations. This demonstrates that the counter example provided in \refsec{sec:appendix-basic} is viable and numerically stable without special care. The \MWUAFTL{} presents low asymptotic regret in this case, which is clearly sub-linear and possibly logarithmic in $T$. The best intermediary $k$-means solution at each step was calculated analytically.

%% file: files/ftl_sub.tex
The previous section showed that there are worst case instances that make \FTL{} perform badly. This section will present experimental results, on synthetic and real data sets, that suggest that \FTL{} performs very well on natural data sets.

Figure \reffig{fig:ftl_vs_sim_mnist} shows the regret halted at $t$ of \FTL{} vs. $\log(t)$ for four different data sets, all of size 10000. The first is a random sample from \textsf{MNIST}, labelled $\mathrm{MNIST}$, treated as a $d=784$ dimensional vector, normalized to a unit diameter box, using $k=10$. The others three are Gaussian Mixture Models ($\mathrm{GMM}$) with 3 Gaussians ($\mathrm{GMM3}$) or 4 Gaussians ($\mathrm{GMM4}$) in two dimensions. The $\mathrm{GMM4}$ case was run with $k=3$, where the 4 Gaussians are well separated (means distance is larger than 3 times the standard deviation) hence labelled $\mathrm{WellSepGMM4}$. The two $\mathrm{GMM3}$ data sets are labelled $\mathrm{WellSepGMM3}$ for the well separated case, and $\mathrm{IllSepGMM3}$ for a case where the Gaussians are ill separated (means distance is 0.7 fraction of the standard deviation). In all cases the best intermediary \kmeans{} solution was calculated using \kmeanspp{} with 300 iterations of local search, hence it is an approximation. The figure demonstrates a linear dependency between the regret and $\log(t)$ in these cases. All the standard deviations are set to $0.1$.

%% file: files/approx_regret.tex
\newcommand{\epsc}{\varepsilon_{\mathsf{c}}}
\newcommand{\epshrd}{\varepsilon_{\mathsf{hrd}}}
\newcommand{\minEdge}{\frac{\epshrd}{2t^3\sqrt{d}}}
\newcommand{\minEdgeT}{\frac{\epshrd}{2T^3\sqrt{d}}}
\newcommand{\minDT}{\frac{\epshrd}{2T^3}}
\newcommand{\minDt}{\frac{\epshrd}{2t^3}}
\newcommand{\diameter}{\Delta}
\newcommand{\mindiameter}{\delta}

\newcommand{\refinementCrtr}[1]{q(#1)}

\newcommand{\maxBranch}{(9\sqrt{d}/\epshrd)^{d} \log (T^3)}
\newcommand{\lme}{\Lambda}
\newcommand{\lmb}{\beta}
\newcommand{\calQ}{\mathcal{Q}}

\newcommand{\coreset}[1]{\mathcal{Q}_{#1}}
\newcommand{\epsstar}{\eps^*}
\newcommand{\tR}{\tilde{R}}

\newcommand{\tree}{\mathcal{T}}
\newcommand{\normloss}{\loss_1}
\newcommand{\unifweight}[2]{u_{#1}^{(#2)}}
\newcommand{\paths}{\mathcal{P}}
\newcommand{\mass}{\calM}
\renewcommand{\th}{^{\mathrm{th}}}

We present an algorithm that aims to minimize approximate regret for the online clustering problem. The algorithm uses three main components
\begin{enumerate}
\item The \emph{Incremental $\epsc-$Coreset} algorithm of \refsec{sub:coreset} that maintains a monotone sequence of sets $\{\coreset{t}\}_{t=0}^T$ such that $\coreset{t}$ contains a weighted $\epsc-$coreset for $X_{1:t}$. 
\item A \emph{Hierarchical Region Decomposition} of \refsec{sub:HRD} corresponding to $\{\coreset{t}\}_{t=0}^T$ and provides a tree structure $\calH_t$ related to $\epshrd$-approximations of \kmeans{}.
\item \MWUA{} for tree structured expert sets, \refsec{sub:mtmw}, referred to as \emph{Mass Tree \MWUA{}} (\MTMW{}), that is given $\calH_t$ at every step, and outputs a choice of $k$ centers.
\end{enumerate}

\begin{figure}
\begin{minipage}[t]{0.50\linewidth}
\vspace{0pt}
 \begin{algorithm}[H]
 \caption{\label{alg:main_alg}$\eps$-Regret Minimization}
 \SetAlgoLined
 \KwIn{$\eps, x_1,\ldots,x_T$}
 $t\gets 1\;
 \epsc, \epshrd \gets \epsstar(\eps,k,d,T)$\;
 Initialize $\calH_0$ and \MTMW{}{}\;
 \For{$t=1$ \KwTo $T$}{
   Obtain $\c{t}$ from \MTMW{}{}\;
   Receive $x_t$ and incur loss $\loss(\c{t},x_t)$\;
   Update $\coreset{t}$ to represent $x_t$\;
   Update $\calH_t$ to represent $\coreset{t}$\;
   Provide \MTMW{}{} with $\calH_t$ and $x_t$\;
 }
   \vspace{34pt}
 \end{algorithm}
\end{minipage}
\begin{minipage}[t]{0.56\linewidth}
\vspace{0pt}
 \begin{algorithm}[H]
 \caption{\label{alg:hrd_update}\HRD{} update step}
 \SetAlgoLined
 \KwIn{$t, \calR_{t-1}, x_t, \epshrd$}
 Let $\refinementCrtr{\cdot}$ be the refinement criteria for $x_t$ at $t$\;
 $\calR_t \gets \emptyset, \quad \calU_t \gets \calR_{t-1}$\;
 \While {$\calU_t \neq \emptyset$}{
  Pick and remove a region $R$ from $\calU_t$\;
  \eIf{$\refinementCrtr{R}$}{
     $\calR_t \gets \calR_t \cup R$
  } {
     Halve $R$ in all dimension, resulting with $H$\;
     $\calU_t \gets \calU_t \cup H$\;
  }
  }
 \end{algorithm}
 
\end{minipage}
\end{figure}

We present our main theorem and provide proof later on.
\begin{theorem}
\label{thm:apx_reg_runtime}
Algorithm~\algref{alg:main_alg} has a regret of 
$$\eps\cdot\opt + O\left( k\sqrt{d^3T}\log\left(\frac{kT^3\sqrt{d}}{\eps^2}\right)\right)$$
and runtime of
$$T \cdot O(\sqrt{d} k^2\log(T)\eps^{-2})^{k(d+O(1))}$$
\end{theorem}

We now continue to describe the different components, and then combine them.
\subsection{\label{sub:coreset}Incremental Coreset}
The \emph{Incremental Coreset} algorithm presented in this subsection receives an unweighted stream of points $X_{1:T}$ one point in each time step and maintains a monotone sequence of sets $\{\coreset{t}\}_{t=0}^T$ such that $\coreset{t}$ contains a weighted $\epsc-$coreset for $X_{1:t}$, for some given parameter $\epsc>0$. Formally, we have the following lemma, whose proof is provided in \refsec{sec:appendix-approx}.

\begin{lemma}
\label{lem:coreset_lemma}
  For any time step $t$, the algorithm described in \ref{sec:coreset_alg} outputs a set of points $\coreset{t}$ such that it contains
  a $(1+\epsc)-$coreset for $X_{1:t}$, which we denote $\chi(X_{1:t})$,
  and has size at most $O(k^2\epsc^{-4} \log^4 T)$. Moreover, we have that $\coreset{t}\subseteq \coreset{t+1}$.
\end{lemma}

\subsubsection{\label{sec:char_kmeans_apx}Maintaining an $O(1)$-Approximation for \kmeans{}}
The first part of our algorithm is to maintain the sets of points $\setS{t}_1,\ldots \setS{t}_{s}$, for any time step $t$, representing possible sets of cluster centers 
where $s = O(\log T)$. We require that at any time,
at least one set, denoted $\apxsol$, induces a bicriteria
$(O(k\log^2 T), O(1))$-approximation to the \kmeans{} problem, namely,
$\apxsol$ contains at most
$O(k \log^2 T)$ centers and its loss is at most some constant
times the loss of the best \kmeans{} solution using at most $k$ centers.
Moreover, for each $i \in [s]$, the sequence $\{\setS{t}_i\}_{t=1}^T$ is an \emph{incremental clustering}: 
First, it must be a monotone sequence, i.e. for any time step $t$, $\setS{t-1}_i \subseteq \setS{t}_i$.
Furthermore, if a data point $x$ of the data stream is 
assigned to a center point $c \in \setS{t}_i$ at time $t$, it remains
assigned to $c$ in any $\setS{\tau}_i$ for $t\le\tau\le T$, i.e. until the end of the algorithm.

Each set which contains more than $O(k \log^2 T)$ is said to be
\emph{inactive} and the algorithm stops adding centers to it.
The remaining sets are said to be active.

To achieve this,
we will use the algorithm of Charikar, O'callaghan and Panigrahy \citep{charikar03},
which we call \charALG{}, and
whose performance guarantees are summarized by the following proposition, that follows immediately from \cite{charikar03}.
\begin{theorem}[Combination of
    Lemma 1 and Corollary 1 in~\citep{charikar03}]
  \label{thm:charikar}
  With probability at least 1/2, at any time $t$,
  one set maintained by \charALG{} is an $O(1)$-approximation
  to the \kmeans{} problem
  which uses at most $O(k \log T)$ centers.
\end{theorem}

\subsubsection{\label{sec:coreset_alg} Maintaining a Coreset for \kmeans{}}
Our algorithm maintains a coreset $\subcoreset_i$ based on each solution $\setS{t}_i$
maintained by \charALG{}. It makes use of coresets through the  
coreset construction introduced by \cite{chen2009coresets} whose properties
are summarized in the following theorem.  
\begin{theorem}[Thm. 5.5/3.6 in~\cite{chen2009coresets}]
  \label{thm:chen}
  Given a set $P$ of $T$ points in a metric space and parameters
  $1 > \epsc > 0$ and $\lambda > 0$, one can compute a weighted
  set $\chenCoreset$ such that $|\chenCoreset| = O(k \epsc^{-2} \log T
  (k \log T + \log(1/\lambda))$
  and $\chenCoreset$ is a $(1+\epsc)$-coreset of $P$ for
  \kmeans{} clustering, with probability  $1 - \lambda$.
\end{theorem}

We now review the coreset construction of Chen.
Given a bicriteria $(\alpha, \beta)$-approximation $\setS{t}_0$ to the \kmeans{}
problem, Chen's algorithm works as follows.
For each center $c \in \setS{t}_0$, consider the points in $P$ whose closest center
in $\setS{t}_0$ is $c$ and proceed as follows. For each $i$, we define
the $i\th{}$ ring of $c$ to be the set of
points of cluster $c$ that are at distance $[2^i, 2^{i+1})$ to $c$.
The coreset construction simply 
samples
$\zeta \ge \beta \epsc^{-4} k \log T$ points among the points whose distance
to $c$ is in the range $[2^i, 2^{i+1})$ (if the number of such points
is below $\zeta$ simply take the whole set).
This ensures that the total number of points in the coreset
is $\alpha k \zeta \log \Delta$, where $\Delta$ is the maximum to minimum
distance ratio (which can be assume to be polynomial in $n$ without loss
of generality).

Our algorithm stores at each time $t$ a set of points $\coreset{t}$ of small
size that contains a $(1+\epsc)$-coreset. Moreover, we have that the
sets $\coreset{t}$ are incremental: $\coreset{t-1} \subseteq \coreset{t}$.

To do so, our algorithm uses the bicriteria approximation algorithm
\charALG{} of Section~\ref{sec:char_kmeans_apx} as follows. For each solution
stored by \charALG{}, the algorithm uses it to compute a coreset via
a refinement of Chen's construction.
Consider first applying Chen's construction to each solution $\setS{t}_i$
maintained by \charALG{}. Since \charALG{} is incremental, whatever
decisions we have made until time $t$, center open and point assignment,
will remain unchanged until the end. Thus, applying Chen's construction
seems possible. The only problem is that for a given set $\setS{t}_i$, a
given center $c \in \setS{t}_i$ and a given ring of $c$, we don't know in advance
how many points are going to end up in the ring and so, what should be
sampling rate so as to sample $\Theta(\zeta)$ elements uniformly.

To circumvent this issue, our algorithm proceeds as follows.
For each set $\setS{t}_i$, for each center $c \in \setS{t}_i$, for each $j$,
the algorithm maintains $\log T$ samples: one for each $2^i$ which
represents a ``guess''
on the number of points in the $j\th{}$ ring that will eventually arrive.
More precisely, for a given time $t$, let $p_t$ be the newly inserted
point. The algorithm then considers each solution $\setS{t}_i$, and
the center $c \in \setS{t}_i$ that is the closest to $p_t$. If $p_t$ belongs
to the $j\th{}$ ring, then $p_t$ is added to the set
$\sample(\setS{t}_i, c, j, u)$ with probability $\zeta/2^u$ for each
$u \in [\log T]$ if $|\sample(\setS{t}_i, c, j, u)| \le 2\zeta$.
Let $\sample(\setS{t}_i)$ denote the union over all center $c \in \setS{t}_i$,
integers $j,u \in  [\log T]$ of $\sample(\setS{t}_i,c,j,u)$.
Let $\coreset{t} = \bigcup_i \sample(\setS{t}_i)$. For more details, the reader is referred to the proof for Lemma~\eqref{lem:coreset_lemma} in \refsec{sec:appendix-approx}.

\subsection{\label{sub:HRD}Hierarchical Region Decomposition}
A \emph{region decomposition} is a partition $\calR = \{R_1,\ldots,R_r\}$
of $[0,1]^d$, each part $R_i$ is referred to as \emph{region}.
A \emph{hierarchical region decomposition} (\HRD) is a sequence of region
decompositions $\{\calR_1,\ldots,\calR_t\}$ such that $\calR_\tau$
is a refinement of $\calR_{\tau-1}$, for all $1 < \tau \le t$. In other
words, for all $1 < \tau \le t$, for all region $R \in \calR_\tau$ there exists
a region $R'\in  \calR_{\tau-1}$  such that $R \subseteq R'$.

As the hierarchical region decomposition $\calH = \{\calR_1,\ldots,\calR_t\}$
only partitions existing regions, it allows us to 
naturally define a tree structure $T_{\calH}$, rather than a DAG. There is a node in
$T_{\calH}$ for each region of each $\calR_\tau$. There is an edge
from the node representing region $R$ to the node representing
region $R'$ if $R \subseteq R'$ and there exists a $\tau$ such that
$R \in \calR_\tau$ and $R' \in \calR_{\tau+1}$. We slightly abuse
notation and refer to the node corresponding to region $R$ by $R$.
The bottom-level region decomposition is the region decomposition induced
by the leaves of the tree.
Moreover, given a hierarchical decomposition
$\calH_t = \{\calR_1,\ldots,\calR_t\}$ and a set of points $S$ of size $k$, we
define the \emph{representative regions of $S$ in $\calH_t$} as a sequence of multisets $\{\tR_{\tau}\}_{\tau=1}^t$ where $\tR_\tau = \{R\in \calR_\tau | \exists s\in S . s\in R\}$ with the correct multiplicity w.r.t. $S$. Note that these correspond to a path in $T_{\calH}$.
We define the \emph{Approximate Centers of $S$ induced by $\calH_t$} as the sequence of multisets $\{\tS_{\tau}\}_{\tau=1}^t$ the consists of the centroids of the representative regions of $S$ in $\calH_t$.

\subsection{Adaptive Grid Hierarchical Region Decomposition}
Given a sequence of points in $\reals^d$,
we describe an algorithm that maintains
a hierarchical region decomposition with $d-$dimensional hypercube regions as follows. 
Let $\epshrd>0$ be a parameter s.t. $\minEdgeT$ is a power of 2. 
We require this in order to define an implicit grid with side length $\minEdgeT$, i.e. diameter $\mindiameter_T = \minDT$, such that it can be constructed from a single region containing the entire space by repeated halving in all dimensions. Denote $\mindiameter_t=\minDt$.
We refer to this implicit grid, along with the region tree structure that corresponds to the this halving process as the \emph{Full Grid} and the \emph{Full Grid Tree}. 
Consider a step $t$, $R\in \calR_t$ and $x\in[0,1]^d$. Denote the diameter of $R$ as $\diameter{R}$, and $\displaystyle r=\min_{p\in R}\twonorm{p-x}$ the distance between $R$ and $x$. Notice that if $x\in R$ then $r=0$. 

We define the \emph{refinement criteria induced by $x$ at time $t$} as $\refinementCrtr{R}$, which takes the value true if and only if the diameter of $R$ is smaller or equal $\max(\epshrd\cdot r/2, \mindiameter_t)$.
At a given time $t$, a new point $x_t$ is received and the hierarchical region decomposition obtained at the end of time $t-1$, $\calH_{t-1}$, is refined using the following procedure, which guarantees that all the new regions satisfy the refinement criteria induced by all the points $X_{1:t}$ at the corresponding insertion times. The pseudocode is given in \algref{alg:hrd_update}.

We now turn to proving \emph{Structural Properties} of the
hierarchical region decomposition $\calH_t$ that the algorithm maintains.
The proof of the following lemma follows immediately from the definition.
\begin{lemma}
  \label{eq:shrinking_lemma}
  Consider the hierarchical region decomposition
  $\{\calR_1,\ldots,\calR_t\}$ produced by
  the algorithm at any time $t$.
  Consider a region $R \in \calR_{t-1}$ and
  let $\Delta R$ be the diameter of region $R$, then the following holds. 
  Either region $R$ belongs to $\calR_t$ or each child region
  of $R$ in $\calR_t$ has diameter at most $\frac{1}{2} \Delta R$.
\end{lemma}
\begin{corollary}
\label{eq:lme}
Consider $\{R_t \in \calR_{t}\}_{t=1}^T$ such that $\forall t:~~R_{t+1}\subseteq R_{t}$, a sequence of nested regions of length $T$. We say that such a sequence cannot be refined more than $\lme$ times, i.e. $|\{t|R_{t+1}\neq R_{t}\}|\le \lme$. Lemma~\eqref{eq:shrinking_lemma} along with the fact that the algorithm does not refine regions with diameter smaller than $\mindiameter_T$ give us that 
$$\lme \le -\log\left(\delta_T/\sqrt{d} \right) = -\log\left(\minEdgeT\right)$$
\end{corollary}
The proof for the following Lemma is provided in \refsec{sec:appendix-approx}.
\begin{lemma}
  \label{eq:bounded_region_count}
  For any stream of length $N$, using the above algorithm, we have that the total number of regions that are added at step $t$ is at most $(9\sqrt{d}/\epshrd)^{d}\log(T^3)$ hence the total amount of regions in $\calR_{N}$ is $N (9\sqrt{d}/\epshrd)^{d} \log (T^3)$. 
\end{lemma}
\begin{corollary}
\label{eq:lmb}
Let $R\in \calR_t$ be any region at step $t$ and $S=\left\{R'\in \calR_{t+1}|R'\subseteq R\right\}$ be the set of regions that refine $R$ in the next time step, then we have that the log max branch $\lmb$ is
$$\displaystyle \lmb = \log \max_{t,R}|S|\le \log \left(\left(\frac{9\sqrt{d}}{\epshrd}\right)^d\log(T^3)\right)$$
Due to Lemma~\eqref{eq:bounded_region_count}. Furthermore for sufficiently large $T$ we have that $\lmb\le d\cdot \lme$
\end{corollary}

We will now present a few properties relating to the approximation of the \kmeans{} problem.
\begin{lemma}
  \label{def:local_BTL}
  Let $\epshrd > 0$.
  Consider an online instance of the weighted \kmeans{} problem where
  a new point and its weight are inserted at each time step. Let $x_{t}$ be the weighted point that arrives at time $t$, such that its weight is bounded by $t$.
  Consider the hierarchical region decomposition with parameter $\epshrd$ produced by the 
  algorithm $\calH_t=  \{\calR_1,\ldots,\calR_t\}$, for some time step $t$.

  Consider two multisets of $k$ centers
  $S = \{c_1,\ldots,c_k\},S' = \{c'_1,\ldots,c'_k\}$ such that for all $i\in\{1\ldots k\}$, $c_i$ and $c'_i$ are contained in the same region of the Region Decomposition of step $t$. Then, the following holds.

  $$\forall 1 \le \tau < t, ~~\loss(S', x_{\tau}) \le (1+\epshrd)\loss(S, x_{\tau}) + \epshrd/\tau^5$$

\end{lemma}
We now extend this lemma to solutions for the \kmeans{} problem.
Given a set of $k$ centers $S = \{c_1,\ldots,c_k\}$ and a hierarchical
region decomposition $\calH = \{\calR_1,\ldots,\calR_t\}$, we
associate a sequence $\{\tS_1, \ldots, \tS_t\}$
of \emph{approximate centers for $S$
  induced by $\calH$} by picking for each $c_i$ the
approximation of $c_i$ induced by $\calH$ at step $t$-- the centroid of the region in $\calR_t$ that contains $c_i$. Note that this is a multiset. 
The next lemma follows directly from applying Lemma~\eqref{def:local_BTL} to these approximate centers, and summing over $t$.

\begin{lemma}
	\label{eq:btl_apx_opt}
  For the optimal set of candidate centers in hindsight $S^*$ and $\tS_{t}$ the approximate centers induced by the Hierarchical Region Distribution at time step $t$, for a weighted stream $X_{1:t}$
	$$ (1-\epshrd)\opt - 2\epshrd \le \sum_{t=1}^T\loss(\tS_{t+1}, x_{t}) \le
	 (1+\epshrd)\opt + 2\epshrd $$
\end{lemma}
As Lemma~\eqref{eq:lme} gives that that $\tS_{t+1}\neq \tS_{t}$ at most $k\cdot \lme$ times (each of the $k$ regions may be refined $\lme$ times), and the loss is bounded by $d$, then along with  Lemma~\eqref{eq:btl_apx_opt} we get the following corollary.
\begin{corollary}
	\label{eq:ftl_apx_opt}
	For the optimal set of candidate centers in hindsight $S^*$ and $\tS_{t}$ the approximate centers induced by the Hierarchical Region Distribution at time step $t$, for an \textbf{unweighted} stream $X_{1:t}$
   $$ \sum_{t'=1}^T\loss(\tS_{t'}, x_{t'}) \le
	(1+\epshrd)\opt+kd \lme+2\epshrd$$
\end{corollary}

\subsection{\label{sub:mtmw}\MTMW{}-- \MWUA{} for Tree Structured Experts}
We present an algorithm which we name \emph{Mass Tree \MWUA{} (\MTMW{})} which obtains low regret in the setting of \emph{Prediction from Expert Advice}, as described in \citep{arora2012multiplicative}, for a set of experts that has the tree structure that will soon follow. The algorithm will be a modification of the \emph{Multiplicative Weights Algorithm} and we will present simple modification to the proof of Theorem (2.1) of \cite{arora2012multiplicative}, to obtain a regret bound.

Let $\normloss(\cdot,\cdot)$ denote a bounded loss function in $[-1,1]$. Consider $\tree_T$, a tree whose leaves are all of depth $T$ and the vertices correspond to expert predictions. The expert set is the set of all paths from the root to the leaves, denoted $\paths(\tree)$. We say that for a path $p=(v_1,\ldots,v_T)\in \paths(\tree)$, the prediction that is associated with $p$ at step $t$ is $v_t$ such that we can write the loss of the path w.r.t. the stream of elements $X_{1:T}$ as $\normloss(p,X_{1:T})=\sum_{t=1}^T\normloss(v_t,x_t)$. 

We associate a \emph{mass} to any vertex in $\tree_T$ as follows.
Define the $\mathrm{mass}$ of the root as $1$, and the mass of any other node $v$, denoted $\mass(v)$, as $\mass(v)=\frac{\mass(v')}{\mathrm{deg}(v')}$, where $v'$ is its parent and $\mathrm{deg}(v')$ is the out degree of $v'$. We define the mass of a path as the mass of the leaf node at the end of the path.
before we move on to prove the regret bound, we provide a useful lemma.

\begin{lemma}[Preservation of Mass]
\label{eq:pres_mass}
Let $v$ be a vertex in the tree, $\tilde{\tree}$ a subtree of $\tree$ with $v$ as root, and $\tilde{V}$ the leaves of $\tilde{\tree}$, then $\displaystyle \mass(v)=\sum_{v' \in\tilde{V}}\mass(v')$
\end{lemma}
We move on to bound the regret of the algorithm.
\begin{theorem}
\label{thm:normalized_mtmw}
	For a tree $\tree_T$ running \MWUA{} over the expert set that corresponds to the paths of the final tree, $\paths(\tree_T)$, is possible even if the tree is known up to depth $t$ at any time step $t$, provided the initial weight for path $p$ is modified to $\mass(p)$. The algorithm has a regret with respect to any path $p\in\paths(\tree_T)$ of
	$$ \sqrt{-T ln(\mass(p))} $$ 
and has a time complexity of $O(|\tree_T|)$, the number of vertices of $\tree_T$.
\end{theorem}
\begin{corollary}
	\label{thrm:mwua_reg}
	\MTMW{} for a loss function bounded by $d$ obtains a regret of $d\sqrt{-T ln(\mass(p))}$ by using a normalized loss function
\end{corollary}

\subsection{Approximate Regret Bound}
We will now combine the three components described above to form the final algorithm. First, we will show the main property of a Hierarchical Region Decomposition that is constructed according to the points that are added to form the sequence $\{\coreset{t}\}_{t=1}^T$, which is analogous to Lemma~\eqref{eq:ftl_apx_opt}. Next we define the \emph{k-tree} structure that corresponds to a Hierarchical Region Decomposition, and show that \MTMW{} performs well on this k-tree. Lastly we show that an intelligent choice of parameters for $\epsc$ and $\epshrd$ allows Algorithm~\algref{alg:main_alg} to obtain our main result, Theorem~\eqref{thm:apx_reg_runtime}.

\begin{lemma}
	\label{eq:coreset_ftl_apx_opt}
	Let $\calH$ be a Hierarchical Region Decomposition with parameter $\epshrd$ that was constructed according to $\{\calQ_{t}\}_{t=1}^T$. 
	For $S^*$ the best $k$ cluster centers in hindsight and $\tS_t$ the approximate centers of $S^*$ induced by $\calH$ we have that
    $$ \sum_{t=1}^T\loss(\tS_{t}, x_{t}) \le (1+\epsc+8(\epshrd + \epsc) k\lme)\opt + dk\lme$$
\end{lemma}

Consider the region tree structure $T_{\calH_t}$ described in \refsec{sub:HRD} that corresponds to the Hierarchical Region Decomposition $\calH_t$ at step $t$ defined in \eqref{eq:coreset_ftl_apx_opt}. We define a \emph{$k$-region tree induced by $\calH_T$} as a level-wise $k$-tensor product of $T_{\calH_T}$, namely, a tree whose vertices at depth $t$ correspond to $k$-tuples of vertices of level $t$ of $T_{\calH_T}$. A directed edge from a vertex $(v_1\otimes\ldots \otimes v_k)$ of level $t$ to vertex $(u_1\otimes\ldots \otimes u_k)$ of level $t+1$ exists iff all the edges $(v_i,u_i)$ exists in $T_{\calH_T}$ for every $i\in 1\ldots k$. We define the \emph{$k$-center tree induced by $\calH_T$} or $k$-tree, as a tree with the same topology as the $k$-region tree, but the vertices correspond to multiset of centroids, rather than tensor products of regions, i.e. the $k$-region tree vertex $(v_1\otimes\ldots \otimes v_k)$ corresponds to the $k$-tree vertex $v=[\mu(v_1),\ldots,\mu(v_k)]$ where $\mu(\cdot)$ is the centroid of the given region. The representative regions of any set $S$ of $k$ cluster centers correspond to a path in the $k$-region tree, and the approximate centers correspond to equivalent path in the $k$-tree. An important thing to note is that Lemma~\eqref{eq:coreset_ftl_apx_opt} proves that there exists a path in the $k$-tree such that the loss of the sequence of approximate centers it contains is close to $\opt$ as described therein.

We will now analyze the run of \MTMW{} on the aforementioned $k$-tree. Let $p^*$ denote the $k$-tree path that corresponds to the best cluster centers in hindsight. 
\begin{lemma}
\label{eq:h_inf_bound}
Using the definitions from Lemmas~\eqref{eq:lme}, \eqref{eq:lmb} we have that
 -$\ln(\mass(p^*)) \le k^2\lme\lmb$
  \end{lemma}

Defining $\frac{\eps^2}{2ak^2\log(T^3\sqrt{d})}\le \epsstar \le \frac{\eps^2}{a k^2\log(T^3\sqrt{d})}$, s.t. $\epsstar=2^i$, where $a\ge {34}^2$ is some constant, gives Theorem~\eqref{thm:apx_reg_runtime}. The complete proof is provided in \refsec{sec:appendix-approx}.

%% file: files/discussion.tex
The online \kmeans{} clustering problem has been studied in a variety of settings. In the setting we use, we have shown that no efficient algorithm with sublinear regret exists, even in the Euclidean setting, under typical complexity theory conjunctures, due to \kmeans{} being $APX$-hard \citep{awasthi2015hardness}. We have presented a no-regret algorithm with runtime that is exponential in $k,d$, showing that the main obstacle in devising these algorithms is computational rather than information-theoretic. 

We have shown that \FTL{} with orcale access to the best clustering so far fails to guarantee sublinear regret in the worst case, but performs very well on natural datasets. 
This opens a door for further study, specifically-- what stability constraints on the data stream, such as well separation of clusters, or data points that are IID using well behaved distributions, allow \FTL{} to obtain logarithmic regret?

We presented an algorithm that obtains $\tilde{O}\left(\sqrt{T}\right)$ approximate regret with a runtime of $O(T(k^2\varepsilon^{-2}d^{1/2}\log(T))^{k(d+O(1))})$ using an adaptive variant of \MWUA{} and an incremental coreset construction, which provides a theoretical upper bound for the approximate regret minimization problem. The next steps in this line of research will involve studying lower bounds for this approximate regret minimization problem and providing simpler algorithms such as \FTL{} with a regularizer fit for purpose. Another extension may reduce the dependency on the dimension by performing dimensionality reduction for the data that preserves \kmeans{} cost of clusters, such as the \emph{Johnson-Lindenstrauss tranformation}.

%% file: files/basic_proofs.tex
\subsection{MWUA}

Let's look at the projection of a set of \centre{}s
$\c{}$ to the set of \emph{sites} $S$. This projection is a multiset defined as
$\Pi(\c{},S)=\{\arg\min_{s\in S}(\twonorm{s-\mu})\;|\mu \in \c{}\}$.
Let $C'=\Pi(\c{},S)$ such that $\c{i}^{'}$ is the projected point corresponding to $C_i$. 
We define the projection infinity distance of $\c{}$ onto $S$ as the maximum distance between these pairs, 
i.e. $\max_{i\in 1\ldots |\c{}|}\twonorm{\c{i}^{'} - \c{i}}$, and denote it as $\Vert \c{i} - S \Vert_\infty$.

The following lemma comes from the folklore.
\begin{lemma}
Let $\mu$ be the centroid (mean) of a set of points $P$, and $\hat{\mu}$ another point in space. Then $\loss(\hat{\mu},P) = |P|\cdot\twonorm{\mu-\hat{\mu}}^2 + \loss(\mu,P)$. 
\end{lemma}
\begin{corollary}
\label{lemma:eps-move}
Let $P$ be a set of points, $\boldsymbol{\mu}$ the set of $k$ optimal centers, and $\hat{\boldsymbol{\mu}}$ $k$ alternative centers such that $\Vert \boldsymbol{\mu}-\hat{\boldsymbol{\mu}} \Vert_\infty \le \epsilon$ then $\loss(\hat{\boldsymbol{\mu}},P) \le |P|\cdot\epsilon^2 + \loss(\boldsymbol{\mu},P)$
\end{corollary}
We now turn to the main theorem
\begin{theorem}
Let $S = \{ i \delta ~|~ 0 \leq i \leq \delta^{-1} \}^d \subset \reals^d$ be
the set of \emph{sites} and let $\mathcal{E} = \{ \c{} \subset S ~|~ |\c{}| = k \}$
be the set of experts ($k$-centers chosen out of the sites). Then, for any $0 <
\alpha \leq 1/4$, with $\delta = T^{-\alpha}$, the \MWUA with the expert set
$\mathcal{E}$ achieves regret $\tilde{O}(T^{1 - 2\alpha})$; the per round
running time is $O(T^{\alpha k d})$. 
\end{theorem}
\begin{proof}
Following section (3.9) in \cite{arora2012multiplicative}. The grid distance $\delta$ results with $n=\frac{1}{\delta^d}$ sites hence $N={n \choose k}$ experts, which is $N=O(\delta^{-kd})$. Denote the regret with respect to the grid experts as the \emph{grid-regret}. Running a single step in \MWUA requires sampling and weight update, taking $O(kdN)$ time and the algorithm guarantees a grid-regret of at most $2\sqrt{ln(N)T}=2\sqrt{kd\ln(\delta^{-1})T}$. 

Let $\c{g}=\Pi(\cstar{T+1},S)$ be the closest grid sites to $\cstar{T+1}$. Because $\Vert c_{T+1}^*-\c{g} \Vert_\infty \le \frac{\sqrt{d}}{2} \delta$, \eqref{lemma:eps-move} gives  $\loss(\c{g},X_{1:T})-\loss(\cstar{T+1},X_{1:T})\le \frac{d\delta^2}{4} T$. Hence the regret of the algorithm is at most $2\sqrt{kd\ln(\delta^{-1})T}+\frac{d\delta^2}{4} T$, so choosing $\delta=T^{-\alpha/2}$ for $\alpha\in (0,\frac{1}{2}]$ yields an algorithm with $\tilde{O}(T^{1-\alpha})$ regret and a per step time complexity of $O(T^{\alpha k d /2})$.
\end{proof}

\subsection{Reduction}
The following is a proof for Theorem~\eqref{thrm:offline_reduction}
\begin{proof}
We will reduce the offline problem with point set $P$ to the online problem by generating a stream $X$ of $T$ uniformly sampled points from $P$ and running $\mathcal{A}$ on $X$; $\mathcal{A}$ generates $T$ intermediate cluster centers $\{\c{t}\}^{T}_{t=1}$, and we return the best one with respect to the entire set $P$.

Denote the best offline \kmeans{} solution as $\cstar{}$-- the optimal clustering for the stream $\cstar{T+1}$ may not coincide with the optimal offline clustering $\cstar{}$, but it must perform at least as good as $\cstar{}$ on $X_{1:T}$. Denote $r$ the internal randomness of $\mathcal{A}$. The regret guarantee gives
\begin{align*}
    \bbE_r [\bbE_X[\sum^T_{t=1}\loss(\c{t},x_t) - \loss(\cstar{T+1},X)]] &\le T^\alpha
\intertext{Using the linearity of expectation and noticing $\cstar{T+1}$ doesn't depend on $r$.}
    \sum^T_{t=1}\bbE_r [\bbE_X[\loss(\c{t},x_t)]] &\le T^\alpha + \bbE_X[\loss(\cstar{T+1},X)] \\
    &\le T^\alpha + \bbE_X[\loss(\cstar{},X)]
\intertext{Using $\forall \c{} : \bbE_{x_t}[\loss(\c{},x_t)]=\frac{\loss(\c{},P)}{|P|}$}
    \sum^T_{t=1} \bbE_r [\bbE_X[\frac{\loss(\c{},P)}{|P|}]] &\le T^\alpha + \frac{T\opt}{|P|}
\intertext{Define $\epsilon_t$ s.t. $\bbE_r[\bbE_X[\loss(\c{},P)]]=(1+\epsilon_t)\opt$. Because $\cstar{}$ is optimal, we know $\forall t : \epsilon_t \ge 0$.}
    \sum^T_{t=1} \frac{(1+\epsilon_t)\opt}{|P|} &\le T^\alpha + \frac{T\opt}{|P|}
\intertext{Rearranging}
    \sum^T_{t=1} \epsilon_t &\le \frac{|P|T^\alpha}{\opt}
\intertext{Denote $\epsilon^*=\min_t(\epsilon_t)$}
     T\cdot \epsilon^* &\le \frac{|P|T^\alpha}{\opt} \\
     \epsilon^* &\le \frac{|P|T^{\alpha-1}}{\opt}
\end{align*}
	So provided $\opt^{-1}=\mathrm{poly}(|P|)$ one can choose $T=\mathrm{poly}(|P|)$ s.t. $\epsilon^*$ is arbitrarily small. $\epsilon^*$ is a non-negative random variable, hence this condition suffices to produce an approximation algorithm with arbitrary $\epsilon$ for \kmeans{}. \cite{awasthi2015hardness} shows that this is \NPHARD{}, finishing the proof.
\end{proof}

\subsection{\FTL{}}
The following is a proof for Theorem~\eqref{thrm:ftl}
\begin{proof}
	We will present an algorithm that generates a stream of points on a line, for $k=2$ and any value of $T$, such that the regret \FTL{} obtains for the stream can be bounded from below by $c\cdot T$ where $c$ is some constant. Extending the result to arbitrary $k$ can be done by contracting the bounding box where the algorithm generates points by a factor of $2k$, and adding data points outside of it in $k-2$ equally spaced locations, to force the creation of $k-2$ clusters, one for each location.

The stream will consist of points in 3 locations $-\delta, 0, (1-\delta)$ for $\delta<\frac{1}{4}$. We will call $\c{t}$ a ($-\delta$)-clustering if it puts all the points at ($-\delta$) in one cluster and the rest in the other cluster, and a ($1-\delta$)-clustering puts all the point at ($1-\delta$) in one cluster and the rest in the other cluster. We will define a stream that has a ($1-\delta$) optimal $\cstar{T+1}$ clustering. 

We will keep the amount of points in $0$ and ($-\delta$) equal up to a difference of 1 at any step by alternative between the two any time we put points in one of them. There exists $n^*=f(\delta)=O(\delta^{-2})$ such that if there is one point at ($1-\delta$) and $n^*$ points in each of $0,-\delta$ then the optimal clustering is the ($1-\delta$)-clustering and for $n^*+1$ points in each of $0,-\delta$ the optimal clustering is the ($-\delta$)-clustering.

Our algorithm will start with a point in ($1-\delta$) making $\c{t}$ a ($-\delta$)-clustering. The next points will be added as follows-- as long as $\c{t}$ is a ($1-\delta$)-clustering add a point to ($-\delta$) or 0, balancing them; if $\c{t}$ has just become a ($-\delta$)-clustering, the next point will be ($1-\delta$), inflicting an loss for \FTL{} which depends only on $\delta$ hence it is $O(1)$, and making $\c{t+1}$ a ($1-\delta$)-clustering again. Halt when we have $T$ points. 

When the algorithm halts a ($1-\delta$)-clustering is either the optimal clustering or is at most $O(1)$ below the optimal clustering, so we will say that $\cstar{T+1}$ is a ($1-\delta$)-clustering without loss of generality. \FTL{} will jump from ($-\delta$)-clustering to ($1-\delta$)-clustering $O(T/n^*)$ times, and suffer a loss of at least the loss $\cstar{T+1}$ incurs at that step (because we balance ($-\delta$) and $0$, \FTL moves the lower cluster away from the middle hence suffer a slightly larger loss than $\cstar{T+1}$ at the next stage). This means that the regret is larger than $O(T/n^*)$ which is bounded from below  by a linear function in $T$ for a fixed $\delta$.
\end{proof}

%% file: files/apx_regret_proofs.tex
\subsection{Incremental Coreset}
The following is a proof for Lemma~\eqref{lem:coreset_lemma}
\begin{proof}
  Note that by the definition of the algorithm, each set
  $\setS{t}_i \in \{\setS{t}_1,\ldots,\setS{t}_s\}$ contains at most $O(k \log^2 T)$
  centers. 
  It follows that, by the definition
  of the algorithm, any $\setS{t}_i \in \{\setS{t}_1,\ldots,\setS{t}_s\}$
  is such that 
  $$|\sample(\setS{t}_i)| =
  \sum_{c\in \setS{t}_i,j \in [\log T],u\in [\log T]} |\sample(\apxsol, c, j, u)|
  \le  O(k \zeta \log^3 T)$$
  The overall bound follows
  from the fact that $s = O(\log T)$. 
  Moreover,
  the fact that  at any time $t$, we have that $\sample(\setS{t}_i) \subseteq \sample(\setS{t'}_i)$ for $t' \ge t$ follows from
  Theorem~\eqref{thm:charikar} and the definition of the algorithm.

  We then argue that $\coreset{t}$ contains a $(1+\epsc)$-coreset.
  Recall that there exists 
  a set $\apxsol \in \{\setS{t}_1,\ldots,\setS{t}_s\}$ that induces a
  bicriteria $(O(\log^2 T), O(1))$-approximation to the \kmeans problem.
  Thus, 
  for a given center $c \in \apxsol$, for any $j$, the $j\th{}$
  ring of $c$ is sampled appropriately (up to a factor 2), for the
  value $u$ which is such that $2^{u-1} \le \log T(c,j,t) < 2^u$ where 
  $n(c,j,t)$ is the total number of points in the $j\th{}$ ring of $c$
  (in solution $\apxsol$) at time $t$.
  Thus, combining this observation with Theorem~\eqref{thm:chen}, we have
  that $\sample(\apxsol) \subseteq \coreset{t}$
  indeed contains a set of points that is
  a $(1+\epsc)$-coreset.
\end{proof}

\subsection{Hierarchical Region Decomposition}
The following is a proof for Lemma~\eqref{eq:bounded_region_count}
\begin{proof}  
At any time step $t$, each point $x_t$ corresponds to a refinement criteria at time $t$ over regions in the \emph{Full Grid Tree}, namely, $\refinementCrtr{\cdot}$, s.t. there exists a frontier (i.e. the leaves of a subtree of the Full Grid Tree) that separates the vertices from above, which have $\refinementCrtr{R}=\mathrm{False}$, and the vertices below have $\refinementCrtr{R}=\mathrm{True}$. This frontier can be thought of as the minimal refinement requirement along each path in the Full Grid Tree that corresponds to adding $x_t$. In order to satisfy the requirements of all the points in the stream and have that all the regions in $\calR_t$ have $\refinementCrtr{R}=\mathrm{True}$ (for the corresponding time steps $t'\le t$) the algorithm iterates the existing frontier in the Full Grid Tree that corresponds to the current region decomposition and extends it to match the requirements due to $X_{1:t-1}$ together with the new minimum requirements introduced by the new point. Hence the algorithm removes a subset of regions from the $\calR_{t-1}$ and replaces them with a subset of the frontier that corresponds to $x_t$. We now turn to prove that the frontier in the Full Grid Tree of any point in space is of size $(9\sqrt{d}/\epshrd)^{d} \log (T^3)$, proving the lemma.

  Let $x$ denote an arbitrary point in space whose frontier size we are trying to bound. The frontier is made up of an area in space close to $x$ that contains only regions of the minimum diameter $\mindiameter_T$, where we use $T$ instead of $t$ as it lower bounds for the diameter. When the distance from $x$ is larger than $r_{hrd}= 2\mindiameter_T/\epshrd$ the dominant term in the refinement criteria is $\epshrd \cdot r/2$, hence this is a hypersphere of radius $r_{hrd}$ centered around $x$.
  
  Outside of the $r_{hrd}$ sphere we will have thick shells with inner radius $r_{hrd}\cdot 2^i$ and outer radius of $r_{hrd}\cdot 2^{i+1}$ where i is some nonnegative integer, where the refinement criteria requires that the maximum diameter will be $\mindiameter_T\cdot 2^i$.
    
  Consider a bounding box for the shell that corresponds to $i\ge 0$. The sides of such a hypercube are of length $2\cdot r_{hrd}\cdot2^{i+1}$. In order to partition this hypercube to regions of diameter $\mindiameter_T\cdot 2^i$ we require $(8\sqrt{d}/\epshrd)^d$ regions. 
  If we iterate the shells from the outermost shell inward, we can assume that a sphere twice as large as shell $i$ was already partitioned to regions with half the resolution of shell $i$ before iterating shell $i$. 
  
  Now we can account for the fact that the spheres, of radius $r$, are not aligned with the regions of the Full Grid Tree for the corresponding resolution/depth. Let us extend the bounding box of shell $i$ such that it is aligned with the full grid in the resolution of the next shell ($i+1$). This means that the edge length of the box is enlarged by at most two units of edge length $\mindiameter_T\cdot 2^i/\sqrt{d}$ from each side, resulting in $(4 + 8\sqrt{d}/\epshrd)^d\le (9\sqrt{d}/\epshrd)^d$ regions, for sufficiently small $\epshrd$. 
  
  There are at most $\log(1/r_{hrd})=\log (2T^3)$ shells, and noticing that the innermost shell accounts for the core, this gives the bound.
\end{proof}

The following is a proof for Lemma~\eqref{def:local_BTL}
\begin{proof}  
  Fix $\tau$ and focus on $c_i$ the cluster center in $S$ closest to $x_{\tau}$, and its corresponding $c'_i$.
  
  consider the region $R$ containing $c_i$.  
  If $R$ also contains $x_{\tau}$ then, since $x_{\tau}$ has been inserted
  to the stream and so $R$ has diameter of at most
  $\mindiameter_t$. Because $c'_i$ is also contained in this region the weighted loss is at most $\tau\cdot (\frac{\epshrd}{2\tau^3})^2$.

  Thus, we turn to the case where $R$ does not contain $x_{\tau}$, hence $||x_{\tau} - c_i|| > 0$.
  Denote $r = ||x_{\tau}-c_i||$.
  By definition of the algorithm and again since
  $x_{\tau}$ has already been inserted to the stream,
  this region must be a hypercube
  with a diameter of at most $\Delta R \le \max(\epshrd \cdot r/2, \frac{\epshrd}{2\tau^3})$
  since otherwise, the region is refined again when $x_{\tau}$ is inserted.
  Cauchy Schwartz gives us
  $$||x_{\tau}-c'_i||^2 \le ||x_{\tau}-c_i||^2+2||x_{\tau}-c_i||\cdot ||c_i-c'_i||+||c_i-c'_i||^2 \le r^2 +r \cdot \Delta R +(\Delta R)^2$$
  If $r \le 1/\tau^3$ then $\Delta R \le \epshrd/2\tau^3$, hence
  \begin{align*}
   ||x_{\tau}-c'_i||^2 &\le ||x_{\tau}-c_i||^2 + \epshrd/\tau^6 
  \intertext{Otherwise, $\epshrd\cdot r/2 > \frac{\epshrd}{2\tau^3}$ hence $\Delta R \le \epshrd \cdot r/2$}
   ||x_{\tau}-c'_i||^2 &\le ||x_{\tau}-c_i||^2(1 + \epshrd)
  \intertext{Hence, accounting for the $\tau$ weight one gets}
   \loss(S',x_{\tau}) &\le \loss(S,x_{\tau}) (1+\epshrd)+ \epshrd/\tau^5
  \end{align*}

\end{proof}

\subsection{\MTMW{}}
The following is a proof for Lemma~\eqref{eq:pres_mass}
\begin{proof}
We will show a proof by induction on the subtree height $h$. For $h=1$ we have $\tilde{V}=\{v\}$ hence the property holds. For $h+1$ we can denote the children of $v$ as $U$ and use the induction hypothesis for each child's subtree and get that 
$$\sum_{v'\in\tilde{V}}\mass(v')=\sum_{v'\in U}\mass(v')=\sum_{v'\in U}\mass(v)/|U|=\mass(v)$$ 
Where the last equality used the recursive mass definition.
\end{proof}

The following is a proof for Theorem~\eqref{thm:normalized_mtmw}
\begin{proof}
For a rooted path $p=(v_1\ldots v_{T})$ define the \emph{uniform weight of $p$ at time $t$} with respect to the stream $X_{1:t}$, as $\unifweight{p}{t}=\prod_{\tau=1}^{t-1}(1-\eta \normloss(v_{\tau},x_{\tau}))$, which corresponds to the weight \MWUA{} with uniform initial weights would associate that expert $p$ before witnessing $x_t$. We extend this notation for any rooted path $p$, as long as it has length at least $t$. In our modified algorithm, the weight associated to expert $p$ at step $t$ is $w^{(t)}_p=\mass(p)\unifweight{p}{t}$.  Denote $p(v)$ the path from the root to vertex $v$.
Using the modified initial weight, the probability \MWUA{} will output the prediction that corresponds to any node $v_t$ at depth $t$ at step $t$, due to choosing some path that contains it, is given by (before normalizing)
	$$ w^{(t)}_{v_t}=\sum_{p\in \paths(\tree_T): v_t\in p}w^{(t)}_p=\sum_{p\in \paths(\tree_T): v_t\in p} \unifweight{p}{t}\cdot \mass(p) = \unifweight{p(v_t)}{t}\cdot \mass(v_t) $$
The weight and mass are functions of known quantities at step $t$, where the equality is from the definition of $\unifweight{p}{t}$ and due to Lemma~\eqref{eq:pres_mass}.
Following the proof of Theorem (2.1) of \cite{arora2012multiplicative}, we define $$\Phi^{(t)}=\sum_{p\in\paths(\tree_T)}w^{(t)}_p \quad\quad (\textbf{m}^{(t)})_p=\normloss(p,x_t) \quad\quad (\textbf{p}^{(t)})_p\propto w_p^{(t)}$$
The potential, the loss vector, and the normalized probability vector, respectively.
For any path $p$, due to the fact $\Phi^{(1)}=1$ we can modify Equation (2.2) in \citep{arora2012multiplicative} to
\begin{align*}
\Phi^{(T+1)}/\mass(p) &\le (1/\mass(p))\exp(-\eta\sum_{t=1}^T\textbf{m}^{(t)}\cdot\textbf{p}^{(t)} )
\intertext{Using the weight of $p$ after the last step as lower bound for $\Phi^{(T+1)}$, we change Equation (2.4) to}
\Phi^{(T+1)}/\mass(p) &\ge w_{p}^{(T+1)}/\mass(p) = \unifweight{p}{T+1}
\end{align*}
Hence the rest of the proof is left intact, where $n$ (the amount of experts) is replaced with $(1/\mass(p))$, so the regret changes to
$$\sqrt{-T\ln(\mass(p)))}$$
The running time is composed of sampling a path which is done by iterating the vertices of depth $t$ and calculating their weights, which finishes the proof.
\end{proof}

\subsection{Combining the Components}
The following is a proof for Lemma~\eqref{eq:coreset_ftl_apx_opt}
\begin{proof}
Denote the times where $\tS_{t}\neq \tS_{t-1}$ as $t_1,t_2\ldots t_{T_0}$, where we consider $t_1=1$ for this purpose, and $T_0$ is the amount of different time steps where this occurs. For ease of notation we denote $t_{T_0+1}=T+1$. Furthermore, we use $X_{[a:b)}$ to denote $X_{a:b-1}$. As $\tS_{t}=\tS_{t-1}$ for any other time step we have that $\tS_{t_{i+1}-1}=\tS_{t_i}$ hence
\begin{align*}
\sum_{t=1}^T\loss(\tS_{t}, x_{t}) =& 
\sum_{i=1}^{T_0}\loss(\tS_{t_i}, X_{[t_i:t_{i+1})})\le 
\sum_{i=1}^{T_0}\loss(\tS_{t_{i+1}-1}, X_{[t_i:t_{i+1})})
\intertext{Excluding the $T_0$ points that resulted with a region refinement from the sum, and using the fact that the loss is bounded by $d$}
\sum_{t=1}^T\loss(\tS_{t}, x_{t}) \le &
\sum_{i=1}^{T_0}\loss(\tS_{t_{i+1}-1}, X_{[t_i:t_{i+1}-1)}) + dT_0\\
\le & \sum_{i=1}^{T_0}\left(\loss(\tS_{t_{i+1}-1}, X_{[1,t_{i+1}-1)})-\loss(\tS_{t_{i+1}-1}, X_{[1,t_{i})})\right) + dT_0
\intertext{As the loss is nonnegative}
\sum_{t=1}^T\loss(\tS_{t}, x_{t}) \le &
\sum_{i=1}^{T_0}\left(\loss(\tS_{t_{i+1}-1}, X_{[1,t_{i+1}-1)}) - \loss(\tS_{t_{i+1}-1}, X_{[1,t_i-1)})\right) +dT_0
\intertext{Using the coreset property}
\sum_{t=1}^T\loss(\tS_{t}, x_{t}) \le &
\sum_{i=1}^{T_0}(1+\epsc)\loss(\tS_{t_{i+1}-1}, \chi(X_{[1,t_{i+1}-1)})) - \\
& \sum_{i=1}^{T_0} (1-\epsc)\loss(\tS_{t_{i+1}-1}, \chi(X_{[1,t_i-1)})) + dT_0
\intertext{As the Hierarchical Region Decomposition $\calH_{t_{i+1}-1}$ was constructed according to a superset of $\chi(X_{1:t_{i+1}-2})$ and $\chi(X_{1:t_i-2})$, one can use Corollary~\eqref{eq:btl_apx_opt} and get}
\sum_{t=1}^T\loss(\tS_{t}, x_{t}) \le &
 \sum_{i=1}^{T_0}(1+\epsc)(1+\epshrd)\loss(S^*, \chi(X_{[1,t_{i+1}-1)})) - \\
& \sum_{i=1}^{T_0}(1-\epsc)(1-\epshrd)\loss(S^*, \chi(X_{[1:t_i-1)})) + dT_0+2\epshrd
\intertext{Now the series telescope, and along with $\epsc,\epshrd<1$ rearranging gives}
\sum_{t=1}^T\loss(\tS_{t}, x_{t}) \le &
\sum_{i=1}^{T_0}((\epsc+\epshrd)\loss(S^*, \chi(X_{[1,t_{i+1}-1)}))) +\loss(S^*, \chi(X_{[1,T)})) + 2dT_0
\intertext{As the loss is nonnegative we can extend the substreams until T, and using the coreset property}
\sum_{t=1}^T\loss(\tS_{t}, x_{t}) \le &
((1+\epsc) +  8 T_0 (\epsc+\epshrd))\loss(S^*, X_{1:T}) + 2 dT_0
\end{align*} 
finally, with Lemma~\eqref{eq:lme} we have that $T_0 \le k\cdot \lme$, which finishes the proof.
\end{proof}

The following is a proof for Lemma~\eqref{eq:h_inf_bound}
  \begin{proof} 
  Using Lemma~\eqref{eq:lme} we can see that $\log(\mathrm{deg}(v_t))$ can take a non zero value at most $k\lme$ times, each is bounded by $k\cdot \lmb$ as each node can be replaced by $(\maxBranch)^k$ children nodes in the $k$-tree.
  \end{proof}

The following is a proof for Theorem~\eqref{thm:apx_reg_runtime}
\begin{proof}
Using the bounds for $\epshrd,\epsc$ we have that
\begin{align*}
\lme=\log(\frac{2T^3\sqrt{d}}{\epshrd})& \le \log(\frac{4adT^6k^2}{\eps^2}) \le 2\log(\frac{akT^3\sqrt{d}}{\eps^2})
\intertext{Next we will show that $\epshrd$ is of the same order as $\eps/k \lme$}
\epshrd \cdot k \lme & \le 
\frac{\eps^2}{ak^2\log(T^3\sqrt{d})} \cdot  2k \log(\frac{a k\sqrt{d} T^3}{\eps^2}) \le \\
& \le \frac{2\eps^2}{ak\log(T^3\sqrt{d})} \log(T^3\sqrt{d})k\log(\frac{a}{\eps^2}) \le 4(\frac{\eps}{\sqrt{a}})^2 \log(\frac{\sqrt{a}}{\eps}) \le \frac{2\eps}{\sqrt{a}}
\end{align*}
As stated in Corollary~\eqref{eq:lmb}, for sufficiently large $T$ we have $\lmb < d \cdot \lme$, hence, along with Lemma~\eqref{eq:h_inf_bound}, we have that \MWUA{} has a regret w.r.t. the best path of 
$k\lme\sqrt{dT}\le 2k\log(\frac{akT^3\sqrt{d}}{\eps^2})\sqrt{dT}$. 
For $a \ge {34}^2$ we get $\frac{2\eps}{\sqrt{a}} \le \frac{\eps}{17}$, which makes Lemma~\eqref{eq:coreset_ftl_apx_opt} bound the $\eps$-Approximate Regret of the best path. hence the final regret of
$$O\left(k\log\left(\frac{akT^3\sqrt{d}}{\eps^2}\right)\sqrt{d^3T}\right) + \eps \cdot \opt $$ 

Furthermore, as $|\calQ_T|\le k^2\log^4(T)\epsc^{-4} = O(k^{10}\log^8(T)\log^4(d)\eps^{-8})$, using Lemma~\eqref{eq:bounded_region_count} with $N=|\calQ_T|$ we can bound the $k$-tree vertices by the amount of leaves times the depth $T$, and have that the runtime is 
$$O(T(|\calQ_T|(32\sqrt{d}/\epshrd)^d\log(T^3))^k)=T \cdot O(k^{10}\log^9(T)\log^4(d)\eps^{-8})^k O(\sqrt{d} k^2\log(T)\eps^{-2})^{dk}$$
\end{proof}